\documentclass{article}

\usepackage[preprint,nonatbib]{Preprint_MMMS}
\usepackage{amssymb}
\usepackage{amsthm}
\usepackage{amsfonts}
\usepackage{amsmath}
\usepackage{graphicx}

\usepackage[utf8]{inputenc} 
\usepackage[T1]{fontenc}    
\usepackage{hyperref}       
\usepackage{url}            
\usepackage{booktabs}       
\usepackage{amsfonts}       
\usepackage{nicefrac}       
\usepackage{microtype}      
\usepackage{xcolor}         
\usepackage{amsmath}
\usepackage{amssymb}
\usepackage{amsthm}
\usepackage{natbib}

\newtheorem{theorem}{Theorem}[section]
\newtheorem{definition}[theorem]{Definition}
\newtheorem{proposition}[theorem]{Proposition}
\newtheorem{lemma}[theorem]{Lemma}

\newtheorem{corollary}[theorem]{Corollary}

\title{Generalization of the Gibbs algorithm with high probability at low temperatures}

\author{%
  Andreas Maurer\\
  Istituto Italiano di Tecnologia, CSML, 16163 Genoa, Italy\\
  \texttt{am@andreas-maurer.eu} \\
}

\begin{document}

\maketitle

\begin{abstract}
The paper gives a bound on the generalization error of the Gibbs algorithm,
which recovers known data-independent bounds for the high temperature range
and extends to the low-temperature range, where generalization depends
critically on the data-dependent loss-landscape. It is shown, that with high probability the
generalization error of a single hypothesis drawn from the Gibbs posterior
decreases with the total prior volume of all hypotheses with similar or
smaller empirical error. This gives theoretical support to the belief in the
benefit of flat minima. The zero temperature limit is discussed and the
bound is extended to a class of similar stochastic algorithms.
\end{abstract}

\section{Introduction}

Controlling the difference between the empirical error and the expected
future error of a hypothesis is a fundamental problem of learning theory.
This paper gives high probability bounds on this generalization gap for
individual hypotheses drawn from the Gibbs posterior. The Gibbs posterior
assigns probabilities, which decrease exponentially with the hypothesis'
empirical error, relative to some prior reference measure. Such
distributions are the minimizers of the PAC-Bayesian bound (\cite%
{mcallester1999pac}) and limiting distributions of stochastic gradient
Langevin dynamics (SGLD, \cite{raginsky2017non}). It has been argued by \cite%
{zhang2018theory}, that the popular method of stochastic gradient descent
(SGD) may also be reinterpreted as a form SGLD, and is thus also related to
the Gibbs-posterior. The Gibbs-algorithm, which generates the posterior from
data, is therefore an important theoretical construction in the study of the
generalization properties of several stochastic algorithms applied to
non-convex learning tasks.

There are various known bounds, both on averages and on single hypotheses
drawn from the posterior (\cite{lever2013tighter}, \cite{raginsky2017non}, 
\cite{kuzborskij2019distribution}, \cite{rivasplata2020pac}, \cite%
{aminian2021exact}, \cite{maurer2024generalization}), but most of these
results become vacuous, when the inverse temperature parameter $\beta $,
which governs the exponential decay of probabilities, exceeds the number $n$
of training examples. For very difficult data, or randomly permuted labels,
this correctly predicts the failure of generalization. For easier data,
however, generalization persists in the low temperature regime $\beta >n$.
This has been experimentally observed for example by \cite{dziugaite2018data}%
, and nicely documented in Figure 1, Section 6 of the respective article.
Any bound which retains explanatory power in the low temperature regime,
must therefore be data- or distribution-dependent.

The bound given here applies with high probability to a single hypothesis
drawn from the Gibbs-posterior. This is an important feature, because after
the laborious processes of SGLD or SGD the final result is the draw of an
individual hypothesis. The bound also predicts better generalization for the
chosen hypothesis, whenever the \textit{total} prior reference volume of
hypotheses with similar or smaller empirical error is large, providing a
partial explanation of the frequent observation, that hypotheses in wide
minima generalize well (\cite{hochreiter1997flat}, \cite{keskar2016large}, 
\cite{wu2017towards}, \cite{zhang2021flatness}).

While recovering and potentially improving on existing results for the high
temperature regime, our result can also guarantee generalization in the zero
temperature limit $\beta \rightarrow \infty $, whenever the set of
hypotheses with minimal empirical risk has positive prior measure. In the
context of binary classification we show that this is the case, whenever the
data has a hard margin and the prior a positive density.\bigskip 

Following a section introducing the necessary notation and definitions the
main result is stated and proved. Section \ref{Section Interpretation}
discusses the implications of this bound in the regimes of high and low
temperature, the zero-temperature limit and the dependence on the underlying
data-distribution. Section \ref{Section Other bounds} gives some more
concrete bounds, Section \ref{Section beyond Gibbs} extends the main result
to more general stochastic algorithms, and Section \ref{Section related work}
summarizes some related literature.

\section{Preliminaries\label{Section preliminaries}}

Throughout the following $\left( \mathcal{X},\Sigma \right) $ is a
measurable space of \textit{data} with probability measure $\mu $. The iid
random vector $\mathbf{x}\sim \mu ^{n}$ is the training sample.

$\left( \mathcal{H},\Omega \right) $ is a measurable space of \textit{%
	hypotheses}, and there is a measurable loss function $\ell :\mathcal{H\times
	X}\rightarrow \left[ 0,\infty \right) $. Members of $\mathcal{H}$ are
denoted $h$ or $g$. We write $L\left( h\right) :=\mathbb{E}_{X\sim \mu }%
\left[ \ell \left( h,x\right) \right] $ and $\hat{L}\left( h,\mathbf{x}%
\right) :=\left( 1/n\right) \sum_{i}\ell \left( h,x_{i}\right) $
respectively for the true (expected) and empirical loss of hypothesis $h\in 
\mathcal{H}$.

The set of probability measures on $\left( \mathcal{H},\Omega \right) $ is
denoted $\mathcal{P}\left( \mathcal{H}\right) $. There is an a-priori
reference measure $\pi \in \mathcal{P}\left( \mathcal{H}\right) $, called
the \textit{prior}. We write $L_{\min }=~$ess $\inf_{h\in \mathcal{H}%
}L\left( h\right) $ and $\hat{L}_{\min }\left( \mathbf{x}\right) =~$ess $%
\inf_{h\in \mathcal{H}}\hat{L}\left( h,\mathbf{x}\right) $, where the
essential infimum refers to the measure $\pi $. We also write $\mathcal{H}%
_{\min }=\left\{ h:L\left( h\right) =L_{\min }\right\} $ and $\widehat{%
	\mathcal{H}}_{\min }\left( \mathbf{x}\right) =\left\{ h:\hat{L}\left( h,%
\mathbf{x}\right) =\hat{L}_{\min }\left( \mathbf{x}\right) \right\} $ for
the respective sets of global minimizers. For $r\in \mathbb{R}$ we also
denote with $\varphi \left( r\right) =\pi \left\{ g:L\left( g\right) \leq
r\right\} $ and $\hat{\varphi}\left( r,\mathbf{x}\right) =\pi \left\{ g:\hat{%
	L}\left( g,\mathbf{x}\right) \leq r\right\} $ the cumulative distribution
functions of the true and empirical loss respectively.

\bigskip
The Gibbs algorithm at inverse temperature $\beta >0$ is the map $\hat{G}%
_{\beta }:\mathbf{x}\in \mathcal{X}^{n}\mapsto \hat{G}_{\beta }\left( 
\mathbf{x}\right) \in \mathcal{P}\left( \mathcal{H}\right) $ defined by%
\begin{equation*}
	\hat{G}_{\beta }\left( \mathbf{x}\right) \left( A\right) =\frac{1}{Z_{\beta
		}\left( \mathbf{x}\right) }\int_{A}e^{-\beta \hat{L}\left( h,\mathbf{x}%
		\right) }d\pi \left( h\right) \text{ for }A\in \Omega \text{.}
\end{equation*}%
$\hat{G}_{\beta }\left( \mathbf{x}\right) $ is called the \textit{%
	Gibbs-posterior, }the normalizing factor 
\begin{equation*}
	Z_{\beta }\left( \mathbf{x}\right) :=\int_{\mathcal{H}}e^{-\beta \hat{L}%
		\left( h,\mathbf{x}\right) }d\pi \left( h\right)
\end{equation*}%
is called the \textit{partition function}. 

We define a probability measure $%
\rho $ on $\mathcal{H}\times \mathcal{X}^{n}$ by 
\begin{equation}
	\rho \left( A\right) =\mathbb{E}_{\mathbf{x}\sim \mu ^{n}}\mathbb{E}_{h\sim 
		\hat{G}_{\beta }\left( \mathbf{x}\right) }\left[ 1_{A}\left( h,\mathbf{x}%
	\right) \right] \text{ for }A\in \Omega \otimes \Sigma ^{\otimes n}.
	\label{Define rho}
\end{equation}%
Then $\mathbb{E}_{\left( h,\mathbf{x}\right) \sim \rho }\left[ \phi \left( h,%
\mathbf{x}\right) \right] =\mathbb{E}_{\mathbf{x}}\mathbb{E}_{h\sim \hat{G}%
	_{\beta }\left( \mathbf{x}\right) }\left[ \phi \left( h,\mathbf{x}\right) %
\right] $ for measurable $\phi :\mathcal{H}\times \mathcal{X}^{n}\rightarrow 
\mathbb{R}$.

The relative entropy of two Bernoulli variables with expectations $p$ and $q$
is denoted 
\begin{equation}
	\kappa \left( p,q\right) =p\ln \frac{p}{q}+\left( 1-p\right) \ln \frac{1-p}{%
		1-q}.  \label{define little kl}
\end{equation}%
\cite{tolstikhin2013pac} give the inversion rule $\kappa \left( p,q\right)
\leq B$ $\implies $ $q-p\leq \sqrt{2pB}+2B$.

\section{A generic generalization bound\label{Section Main result}}

In applications the otherwise arbitrary function $F$ in the following
theorem is a place-holder for a scalar multiple of the generalization gap.

\begin{theorem}
	\label{Theorem Main}Let $F:\mathcal{H}\times \mathcal{X}^{n}\rightarrow 
	\mathbb{R}$ be some measurable function and $\delta >0$. Then with
	probability at least $1-\delta $ in $\mathbf{x}\sim \mu ^{n}$ and $h\sim 
	\hat{G}_{\beta }\left( \mathbf{x}\right) $ 
	\begin{equation*}
		F\left( h,\mathbf{x}\right) \leq \inf_{r\in \mathbb{R}}\beta r+\ln \frac{1}{%
			\hat{\varphi}\left( \hat{L}\left( h,\mathbf{x}\right) +r,\mathbf{x}\right) }%
		+\ln \mathbb{E}_{\mathbf{x}}\mathbb{E}_{g\sim \pi }\left[ e^{F\left( g,%
			\mathbf{x}\right) }\right] +\ln \left( 1/\delta \right) .
	\end{equation*}
\end{theorem}

\begin{proof}
	By Markov's inequality (Appendix \ref{Section Markov inequality}, Lemma \ref%
	{Markov inequality} (i)) for any real random variable $Y$ 
	\begin{equation*}
		\Pr \left\{ Y>\ln \mathbb{E}\left[ e^{Y}\right] +\ln \left( 1/\delta \right)
		\right\} \leq \delta .
	\end{equation*}%
	We apply this to the random variable $Y=F\left( h,\mathbf{x}\right) +\beta 
	\hat{L}\left( h,\mathbf{x}\right) +\ln Z_{\beta }\left( \mathbf{x}\right) $
	on the probability space $\left( \mathcal{H\times X}^{n},\Omega \otimes
	\Sigma ^{\otimes n},\rho \right) $ as defined in (\ref{Define rho}).
	Together with the definition of the Gibbs-posterior this gives, with
	probability at least $1-\delta $ in $\left( h,\mathbf{x}\right) \sim \rho $
	(equivalent to saying $\mathbf{x}\sim \mu ^{n}$ and $h\sim \hat{G}_{\beta
	}\left( \mathbf{x}\right) $),%
	\begin{align*}
		& F\left( h,\mathbf{x}\right) +\beta \hat{L}\left( h,\mathbf{x}\right) +\ln
		Z_{\beta }\left( \mathbf{x}\right) \\
		& \leq \ln \mathbb{E}_{\mathbf{x}}\mathbb{E}_{g\sim \hat{G}_{\beta }\left( 
			\mathbf{x}\right) }\left[ e^{F\left( g,\mathbf{x}\right) +\beta \hat{L}%
			\left( g,\mathbf{x}\right) +\ln Z_{\beta }\left( \mathbf{x}\right) }\right]
		+\ln \left( 1/\delta \right) \\
		& =\ln \mathbb{E}_{\mathbf{x}}\mathbb{E}_{g\sim \pi }\left[ e^{F\left( g,%
			\mathbf{x}\right) +\beta \hat{L}\left( g,\mathbf{x}\right) +\ln Z_{\beta
			}\left( \mathbf{x}\right) -\beta \hat{L}\left( g,\mathbf{x}\right) -\ln
			Z_{\beta }\left( \mathbf{x}\right) }\right] +\ln \left( 1/\delta \right) \\
		& =\ln \mathbb{E}_{\mathbf{x}}\mathbb{E}_{g\sim \pi }\left[ e^{F\left( g,%
			\mathbf{x}\right) }\right] +\ln \left( 1/\delta \right) .
	\end{align*}%
	Subtract $\beta \hat{L}\left( h,\mathbf{x}\right) +\ln Z_{\beta }\left( 
	\mathbf{x}\right) $ to get%
	\begin{equation}
		F\left( h,\mathbf{x}\right) \leq -\beta \hat{L}\left( h,\mathbf{x}\right)
		-\ln Z_{\beta }\left( \mathbf{x}\right) +\ln \mathbb{E}_{\mathbf{x}}\mathbb{E%
		}_{g\sim \pi }\left[ e^{F\left( g,\mathbf{x}\right) }\right] +\ln \left(
		1/\delta \right) .  \label{Main inequality final form}
	\end{equation}
	
	For any $r\in \mathbb{R}$ and $h\in \mathcal{H}$ we can lower bound the
	partition function by%
	\begin{eqnarray*}
		Z_{\beta }\left( \mathbf{x}\right) &\geq &\int_{\left\{ g:\hat{L}\left( g,%
			\mathbf{x}\right) \leq \hat{L}\left( h,\mathbf{x}\right) +r\right\}
		}e^{-\beta \hat{L}\left( g,\mathbf{x}\right) }d\pi \left( g\right) \\
		&\geq &e^{-\beta \left( \hat{L}\left( h,\mathbf{x}\right) +r\right) }\hat{%
			\varphi}\left( \hat{L}\left( h,\mathbf{x}\right) +r,\mathbf{x}\right) .
	\end{eqnarray*}%
	It follows that%
	\begin{equation*}
		-\beta \hat{L}\left( h,\mathbf{x}\right) -\ln Z_{\beta }\left( \mathbf{x}%
		\right) \leq \inf_{r\in \mathbb{R}}\beta r+\ln \frac{1}{\hat{\varphi}\left( 
			\hat{L}\left( h,\mathbf{x}\right) +r,\mathbf{x}\right) }.
	\end{equation*}%
	Substitution in (\ref{Main inequality final form}) completes the
	proof.\bigskip
\end{proof}

The first step in the proof, the application of Markov's inequality,
produces at once a disintegrated PAC-Bayesian bound (like Theorem 1 (i) of 
\cite{rivasplata2020pac}) as applied to the Gibbs posterior.

The second step in the proof, the lower bound on the partition function with
the cumulative distribution function of the empirical loss, weakens this
result, but serves the purpose of interpretability. In Section \ref{Section
	beyond Gibbs} this simple method is extended to other data-dependent
distributions.

A similar result to Theorem \ref{Theorem Main} is given by \cite%
{viallard2024leveraging}, with the principal difference that the parameter $%
r $ becomes the difference $\hat{L}\left( h^{\prime },\mathbf{x}\right) -%
\hat{L}\left( h,\mathbf{x}\right) $, where $h^{\prime }$ is some test
hypothesis, and $\hat{\varphi}\left( \hat{L}\left( h,\mathbf{x}\right) +r,%
\mathbf{x}\right) $ is equated as part of the confidence parameter $\delta $
and combined with (\ref{Main inequality final form}) in a union
bound.\bigskip

To simplify the statement of some corollaries we introduce the hypothesis-
and data-dependent complexity measure $\Lambda _{\beta }:\mathcal{H\times X}%
^{n}\rightarrow \left[ 0,\infty \right) $%
\begin{equation}
	\Lambda _{\beta }\left( h,\mathbf{x}\right) :=\inf_{r\in \mathbb{R}}\beta
	r+\ln \frac{1}{\hat{\varphi}\left( \hat{L}\left( h,\mathbf{x}\right) +r,%
		\mathbf{x}\right) },  \label{Define Lambda}
\end{equation}%
so the inequality in Theorem \ref{Theorem Main} can be written%
\begin{equation}
	F\left( h,\mathbf{x}\right) \leq \Lambda _{\beta }\left( h,\mathbf{x}\right)
	+\ln \mathbb{E}_{\mathbf{x}}\mathbb{E}_{g\sim \pi }\left[ e^{F\left( g,%
		\mathbf{x}\right) }\right] +\ln \left( 1/\delta \right) .
	\label{Main bound short version}
\end{equation}

For a first application let the loss $\ell $ have values in $\left[ 0,1%
\right] $ and set $F\left( h,\mathbf{x}\right) =n~\kappa \left( \hat{L}%
\left( h,\mathbf{x}\right) ,L\left( h\right) \right) $, with $\kappa $ the
relative entropy as in (\ref{define little kl}). The two expectations above
can be interchanged, and from Theorem \textbf{\ref{Theorem Main}} of \cite%
{maurer2004note} we get $\mathbb{E}_{\mathbf{x}}\left[ e^{n~\kappa \left( 
	\hat{L}\left( h,\mathbf{x}\right) ,L\left( h\right) \right) }\right] \leq 2%
\sqrt{n}$. Substitution in Theorem \ref{Theorem Main} and division by $n$
then give the following.

\begin{corollary}
	\label{Corollary kl}Assume that $\ell $ has values in $\left[ 0,1\right] $
	and let $\delta >0$ and $n\geq 8$. Then with probability at least $1-\delta $
	in $\mathbf{x}\sim \mu ^{n}$ and $h\sim \hat{G}_{\beta }\left( \mathbf{x}%
	\right) $%
	\begin{equation*}
		\kappa \left( \hat{L}\left( h,\mathbf{x}\right) ,L\left( h\right) \right)
		\leq \frac{1}{n}\left( \Lambda _{\beta }\left( h,\mathbf{x}\right) +\ln
		\left( \frac{2\sqrt{n}}{\delta }\right) \right) .
	\end{equation*}
\end{corollary}

Using the inversion rule for $\kappa $ this inequality implies%
\begin{equation*}
	L\left( h\right) -\hat{L}\left( h,\mathbf{x}\right) \leq \sqrt{\frac{2\hat{L}%
			\left( h,\mathbf{x}\right) }{n}\left( \Lambda _{\beta }\left( h,\mathbf{x}%
		\right) +\ln \left( \frac{2\sqrt{n}}{\delta }\right) \right) }+\frac{2}{n}%
	\left( \Lambda _{\beta }\left( h,\mathbf{x}\right) +\ln \left( \frac{2\sqrt{n%
	}}{\delta }\right) \right) .
\end{equation*}%
Ignoring the logarithmic term this gives an approximate rate of $\Lambda
_{\beta }\left( h,\mathbf{x}\right) /n$ for small $\hat{L}\left( h,\mathbf{x}%
\right) $.

This is not the only bound which can be derived from Theorem \ref{Theorem
	Main}. Results for unbounded losses (sub-Gaussian or sub-exponential) are
given in Section \ref{Section Other bounds}. We conclude this section with
some superficial remarks on the quantity $\Lambda _{\beta }\left( h,\mathbf{x%
}\right) $ and Theorem \textbf{\ref{Theorem Main}}.

1. Without infimum the right hand side of (\ref{Define Lambda}) is infinite
for $r<-\hat{L}\left( h,\mathbf{x}\right) $. As $r$ ranges from $-\hat{L}%
\left( h,\mathbf{x}\right) $ to $+\infty $, the first term increases from $%
r=-\beta \hat{L}\left( h,\mathbf{x}\right) $ to $+\infty $, while the second
tern is non-increasing and descends from $+\infty $ to zero. The infimum is
finite and approximated (or attained) in the interval $\left( -\hat{L}\left(
h,\mathbf{x}\right) ,+\infty \right) $.

2. $\Lambda _{\beta }\left( h,\mathbf{x}\right) $ is a random variable in
its dependence on $\mathbf{x}\sim \mu ^{n}$ and $h\sim \hat{G}_{\beta
}\left( \mathbf{x}\right) $. It is non-increasing in $\hat{L}\left( h,%
\mathbf{x}\right) $, which makes some intuitive sense, since for finite $%
\beta $ we can sample arbitrarily large values of $\hat{L}\left( h,\mathbf{x}%
\right) $ within the range of the loss function. If $\hat{L}\left( h,\mathbf{%
	x}\right) $ is larger, we expect the gap to the true loss $L\left( h\right) $
to be smaller.

3. The distributions generated by commonly used stochastic algorithms 
\textit{only approximate} the Gibbs-posterior. Suppose we draw from an
approximating sequence $\zeta _{m}\left( \mathbf{x}\right) \in \mathcal{P}%
\left( \mathcal{H}\right) $ instead of $\hat{G}_{\beta }\left( \mathbf{x}%
\right) $, and $\Delta _{\mathcal{F}}\left( \zeta _{m},\hat{G}_{\beta
}\left( \mathbf{x}\right) \right) \rightarrow 0$ for some integral
probability metric $\Delta _{\mathcal{F}}$ defined by some function class $%
\mathcal{F}$ (Definition \ref{Definition Integral probability metric} in
Appendix \ref{Section Markov inequality}). If there exists $\gamma <\infty $
such $\gamma ^{-1}\exp \left( F\left( .,\mathbf{x}\right) +\beta \hat{L}%
\left( .,\mathbf{x}\right) \right) \in \mathcal{F}$, then we retain from
Lemma \ref{Markov inequality} (ii) (in Appendix \ref{Section Markov
	inequality}) that in the limit $m\rightarrow \infty $ we get the same bound
on $F\left( .,\mathbf{x}\right) $ as in Theorem \textbf{\ref{Theorem Main}}.

\section{Interpretation of Theorem \textbf{\protect\ref{Theorem Main}}\label%
	{Section Interpretation}}

The most important consequence of Theorem \ref{Theorem Main} is an at first
puzzling cooperative phenomenon: generalization of a randomly chosen
individual hypothesis benefits from the \textit{total} prior volume of
hypotheses with similar, or smaller empirical loss. The situation is
depicted in Figure \ref{Landscape picture}.
\begin{figure}[h!]
	\centering
	\includegraphics[width=0.7\textwidth]{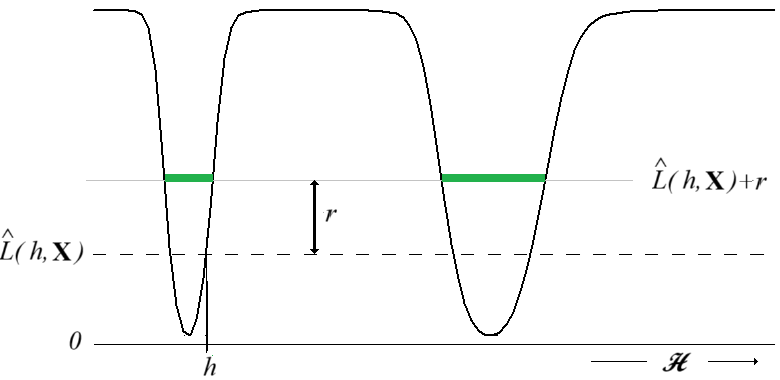}  
	\caption{Schematic representation of the loss landscape, with the prior being the length of horizontal intervals. $h$ is drawn from the Gibbs-posterior and the total length of the thick green lines contributes to the prior volume and thus to generalization. Notice that for large $\beta$ and large ${\hat L}(h,\bf X)$ the optimal $r$ can also be negative.}
	\label{Landscape picture}
\end{figure}
\subsection{Wide minima}

Theorem \textbf{\ref{Theorem Main}} predicts better generalization, if the
near minimal hypotheses, when averaged over the prior, have larger volume.
Formally, if $r^{\ast }$ is the minimizer in the definition of $\Lambda $,
and the $C_{i}\subseteq \mathcal{H}$ are disjoint components of $\left\{ g:%
\hat{L}\left( g,\mathbf{x}\right) \leq \hat{L}\left( h,\mathbf{x}\right)
+r^{\ast }\right\} $, then the $C_{i\text{ }}$maximizing $\pi \left(
C_{i}\right) $ make the greatest contribution to generalization, because $%
\Lambda _{\beta }\left( h,\mathbf{x}\right) =\beta r^{\ast }-\ln \sum_{i}\pi
\left( C_{i}\right) $. If $\mathcal{H}$ is parametrized by $\mathbb{R}^{d}$
with smooth loss, and the $C_{i}$ correspond to basins of attraction of
multiple minima, then $\pi \left( C_{i}\right) $ can be interpreted as
"flatness" or "width". If, for example, $\pi $ is an isotropic Gaussian of
width $\sigma $, and $d$ is very large, then $\pi $ is exponentially
concentrated on a sphere $\mathcal{S}$ of radius $\sigma \sqrt{d}$ (see \cite%
{vershynin2018high}), and $\pi \left( C_{i}\right) $ approximates $\chi
\left( C_{i}\cap \mathcal{S}\right) $, where $\chi $ is the uniform
distribution on the $d-1$-dimensional submanifold $\mathcal{S}$. The minima
corresponding to $C_{i}$ are then "wide" or "narrow" in the sense of this
uniform distribution. Notice that by the definition of the Gibbs posterior $%
\mathcal{S}$ is also an approximation of the effective hypothesis space.

While this is specific to the Gibbs algorithm, it still corroborates to some
extent the empirically supported belief that "wide" or "flat" minima in
non-convex loss landscapes are good for generalization (\cite%
{hochreiter1997flat}, \cite{keskar2016large}, \cite{zhang2018theory}, \cite%
{dziugaite2020search}, \cite{iyer2023wide}).

It has been argued by several authors (\cite{dinh2017sharp}, \cite%
{granziol2020flatness}) that reparametrizations of wide minima may become
very narrow, but still compute the same function, so good generalization
cannot truly be a property of wide minima. If the hypotheses are viewed in
isolation, this seems a valid argument, and several authors thought of
reparametrization-invariant definitions of "width" (see \cite%
{andriushchenko2023modern} and references therein). From the perspective of
the Gibbs algorithm, however, any global reparametrization $T:\mathbb{R}%
^{d}\rightarrow \mathbb{R}^{d}$ must be accompanied by a corresponding
push-forward of the prior $\pi \mapsto T_{\#}\left( \pi \right) $ where $%
T_{\#}\left( \mu \right) \left( A\right) =\pi \left( T^{-1}\left( A\right)
\right) $, and in terms of the new (possibly very singular looking) prior
the neighborhoods of all minima are just as wide or narrow as before the
reparametrization.

\subsection{High temperatures}

If the loss is bounded, say $\ell \leq 1$, then setting $r=1$ instead of the
infimum in (\ref{Define Lambda}) causes the second term to vanish, so $%
\Lambda \left( h,\mathbf{x}\right) \leq \beta $. Thus Corollary \ref%
{Corollary kl} guarantees the following often considerably weaker bound.

\begin{corollary}
	\label{Corollary high temperature}Assume $\ell \leq 1$ and let $n\geq 8$ and 
	$\delta >0$. Then with probability at least $1-\delta $ in $\mathbf{x}\sim
	\mu ^{n}$ and $h\sim \hat{G}_{\beta }\left( \mathbf{x}\right) $%
	\begin{equation*}
		\kappa \left( \hat{L}\left( h,\mathbf{x}\right) ,L\left( h\right) \right)
		\leq \frac{1}{n}\left( \beta +\ln \frac{2\sqrt{n}}{\delta }\right) .
	\end{equation*}
\end{corollary}

This is a data-independent worst-case bound. In the high temperature regime $%
\beta \ll n$ this bound is non-vacuous and comparable to the existing bounds
in \cite{lever2013tighter}, \cite{raginsky2017non}, \cite{dziugaite2018data}%
, \cite{kuzborskij2019distribution}, \cite{rivasplata2020pac} or \cite%
{maurer2024generalization}.

\subsection{Low temperatures\label{Section low temperatures}}

In the low temperature regime, when we cannot fall back on the high
temperature bound of Corollary \ref{Corollary high temperature},
generalization may succeed or fail in a data-dependent way, depending
crucially on the cumulative distribution function of the empirical loss $%
\hat{\varphi}\left( r,\mathbf{x}\right) =\pi \left\{ g:\hat{L}\left( g,%
\mathbf{x}\right) <r\right\} $.

To get an idea of the orders of magnitude involved, consider a data-set like
MNIST, where we can obtain a test error as small as $10^{-2}$ from $10^{4}$
training examples. Let $\beta \approx 10^{5}$, so we are in the
low-temperature regime. According to Theorem \textbf{\ref{Theorem Main}} the
optimal $r^{\ast }$ must then be about $r^{\ast }\approx 10^{-3}$, but for
any $\pi $ this means also that $\ln \left( 1/\hat{\varphi}\left( 10^{-3},%
\mathbf{x}\right) \right) /10^{4}\approx 10^{-2}$, equivalent to $\pi
\left\{ g:\hat{L}\left( g,\mathbf{x}\right) <10^{-3}\right\} \approx
e^{-100}\approx 4\times 10^{-44}$. On the one hand this shows that good
generalization does not make excessive demands on the prior mass of good
hypotheses, which is comforting because it makes the bound seem realistic.
On the other hand it also shows, that it is practically impossible to
estimate $\pi \left\{ g:\hat{L}\left( g,\mathbf{x}\right) <r\right\} $ by
simple trials of $\pi $. This is a drawback of the proposed bound: despite
the fact that it is completely data-dependent, it seems incomputable.

Corollary \ref{Corollary high temperature} would require $\beta \approx
10^{2}$ for a generalization gap of $10^{-2}$, implying a very large
training error as documented in Figure 1, Section 6 of \cite%
{dziugaite2018data}. The high temperature bound cannot explain this figure,
while Theorem \ref{Theorem Main} seems to be consistent with it.

If the prior volume of the set of global empirical minimizers is positive,
then, since almost surely $\hat{\varphi}\left( L\left( h,\mathbf{x}\right) ,%
\mathbf{x}\right) =\pi \left\{ g:\hat{L}\left( g,\mathbf{x}\right) \leq
L\left( h,\mathbf{x}\right) \right\} \geq \pi \left( \widehat{\mathcal{H}}%
_{\min }\left( \mathbf{x}\right) \right) $, we may set $r=0$ in (\ref{Define
	Lambda}), which yields the following.

\begin{corollary}
	\label{Corollary general} If $\pi \left( \widehat{\mathcal{H}}_{\min }\left( 
	\mathbf{x}\right) \right) >0$, then for any $\beta >0$ we have%
	\begin{equation*}
		\Lambda _{\beta }\left( h,\mathbf{x}\right) \leq \ln \left( 1/\pi \left( 
		\widehat{\mathcal{H}}_{\min }\left( \mathbf{x}\right) \right) \right) ,
	\end{equation*}%
	and for $n>8$ and $\delta >0$ with probability at least $1-\delta $ in $%
	\mathbf{x}\sim \mu ^{n}$ and $h\sim \hat{G}_{\beta }\left( \mathbf{x}\right) 
	$%
	\begin{equation*}
		\kappa \left( \hat{L}\left( h,\mathbf{x}\right) ,L\left( h\right) \right)
		\leq \frac{1}{n}\left( \ln \frac{1}{\pi \left( \widehat{\mathcal{H}}_{\min
			}\left( \mathbf{x}\right) \right) }+\ln \frac{2\sqrt{n}}{\delta }\right) .
	\end{equation*}
\end{corollary}

Proposition \ref{Proposition Margins} (i) in the next section shows, that in
the case of binary classification a hard margin $m_{0,\mathbf{x}}^{\ast }>0$
and appropriate alignment of $\pi $ guarantee $\pi \left( \widehat{\mathcal{H%
}}_{\min }\left( \mathbf{x}\right) \right) =\varphi \left( 0,\mathbf{x}%
\right) >0$, so this corollary applies.

\subsection{Margins in binary classification}

When $\beta >n$ then generalization of the Gibbs algorithm works for "good"
data and fails for "bad" data, for which Theorem \ref{Theorem Main} provides
a rough definition: if there are more hypotheses with small empirical error
in data-set $\mathbf{x}$ than in data-set $\mathbf{x}^{\prime }$, then $%
\mathbf{x}$ is better than $\mathbf{x}^{\prime }$. This criterium is simple.
For a given hypothesis space and a given prior it depends only on the
data-set and it guarantees better generalization for better data. But it is
not very intuitive. To provide more intuition we relate the cumulative
distribution function $\hat{\varphi}$ in Theorem \ref{Theorem Main} to the
concept of \textit{margin} in binary classification, which is more in line
with classical approaches to machine learning (\cite{Anthony99}, \cite%
{cristianini2000introduction}).

Suppose that the data consist of labeled inputs, $\mathcal{X}=\mathcal{Z}%
\times \left\{ -1,1\right\} $ with $x=\left( z,y\right) $ and $\mathcal{H}%
\subseteq \mathbb{R}^{d}$. There is a fixed function $\Phi :\mathcal{H}%
\times \mathcal{Z}\rightarrow \mathbb{R}$, such that $\Phi \left( .,z\right) 
$ is continuous on $\mathcal{H}$ for every $z\in \mathcal{Z}$. We also
assume that $\mathcal{H}$ has the following bias- or translation property:
for every $h\in \mathcal{H}$ and $s\in \mathbb{R}$ there is some hypothesis $%
h_{s}\in \mathcal{H}$ such that $\Phi \left( h_{s},z\right) =\Phi \left(
h,z\right) +s$ for every $z\in \mathcal{Z}$. We consider both the 0-1 loss 
\begin{equation*}
	\ell _{\text{01}}\left( h,x\right) =\ell _{\text{01}}\left( h,\left(
	z,y\right) \right) =\left\{ 
	\begin{array}{cc}
		0 & \text{if }\Phi \left( h,z\right) y>0 \\ 
		1 & \text{otherwise}%
	\end{array}%
	\right. .
\end{equation*}%
The case $\mathcal{H=S}^{d-1}\times \mathbb{R}$, $\Phi \left( \left(
u,b\right) ,z\right) =\left\langle u,z\right\rangle -b$ corresponds to
linear classification, where the unit vector $u$ defines the orientation of
a hyperplane, and $b$ its translation. In another important case $\mathcal{H}%
=\mathbb{R}^{d}$, $\Phi \left( w,z\right) $ is the output of a neural
network with input $z$ and parameter vector $w$, where $w$ also includes a
bias after the last layer.

For given $r\geq 0$ and data-set $\mathbf{x}\in \mathcal{X}^{n}$ define the
margin function $m_{r,\mathbf{x}}:\mathcal{H}\rightarrow \mathbb{R}$ by 
\begin{equation*}
	m_{r,\mathbf{x}}\left( h\right) =\max_{I\subseteq \left[ n\right]
		:\left\vert I\right\vert \geq \left( 1-r\right) n}\min_{i\in I}\Phi \left(
	h,z_{i}\right) y_{i}.
\end{equation*}%
Then $m_{0,\mathbf{x}}$ is the usual hard margin, $m_{r,\mathbf{x}}$ is a
soft margin, allowing an error fraction $r$. Let $m_{r,\mathbf{x}}^{\ast
}=\sup_{h\in \mathcal{H}}m_{r,\mathbf{x}}\left( h\right) $.

\begin{proposition}
	\label{Proposition Margins}Let $\mathbf{x}^{\prime }\in \mathcal{X}^{n}$, $%
	r\geq 0$.
	
	(i) If $m_{r,\mathbf{x}}^{\ast }>0$ then $\left\{ g:\hat{L}\left( g,\mathbf{x%
	}\right) \leq r\right\} $ contains a nonempty open subset $O$ of $\mathcal{H}
	$, and if $\pi $ has a nonzero density w.r.t. Lebesgue measure, then $\hat{%
		\varphi}\left( r,\mathbf{x}\right) >0$.
	
	(ii) If $\mathbf{x}^{\prime }\in \mathcal{X}^{n}$ and $m_{r\mathbf{x}%
		^{\prime }}\left( h\right) \leq m_{r\mathbf{x}}\left( h\right) $ for all $%
	h\in \mathcal{H}$ with $\hat{L}\left( h,\mathbf{x}^{\prime }\right) \leq r$,
	then $\hat{\varphi}\left( r,\mathbf{x}^{\prime }\right) \leq \hat{\varphi}%
	\left( r,\mathbf{x}\right) $.
\end{proposition}

Part (i) and Corollary \ref{Corollary general} show that the existence of a
hard margin $m_{0,\mathbf{x}}^{\ast }>0$ implies generalization of the Gibbs
algorithm for all values of $\beta $, whenever the prior has a positive
density. The same is easily shown for the hinge loss $\ell _{\gamma }\left(
h,\left( z,y\right) \right) :=\max \left\{ 0,1-\gamma ^{-1}\Phi \left(
h,z\right) y\right\} $ whenever $m_{0,\mathbf{x}}^{\ast }>\gamma $. Part
(ii) roughly shows that Theorem \ref{Theorem Main} gives better bounds for
uniformly better margins. The uniform monotonicity condition in part (ii) is
quite strong, and in the non-linear case the existence of such an $\mathbf{x}%
^{\prime }$ does not seem guaranteed. In linear classification, however, if $%
m_{r,\mathbf{x}}^{\ast }>0$, there always exists $\mathbf{x}^{\prime }$ such
that $m_{r\mathbf{x}^{\prime }}\left( h\right) \leq m_{r\mathbf{x}}\left(
h\right) $ for all $h\in \mathcal{H}$, obtained by moving the support
vectors towards the maximal margin hyperplane.

\begin{proof}
	$\hat{L}\left( h,\mathbf{x}\right) \leq r$ $\iff $ $h$ makes at most $rn$
	errors $\iff $ there is a set $I\subseteq \left[ n\right] $ such that $%
	\left\vert I\right\vert \geq \left( 1-r\right) n$ and $\min_{i\in I}\Phi
	\left( h,z_{i}\right) y_{i}>0$ $\iff $ if $m_{r,\mathbf{x}}\left( h\right)
	>0 $ $\iff $ $h\in m_{r,\mathbf{x}}^{-1}\left( \left( 0,\infty \right)
	\right) =m_{r,\mathbf{x}}^{-1}\left( \left( 0,m_{r,\mathbf{x}}^{\ast }\right]
	\right) $, where the last identity follows from maximality. We have shown
	that%
	\begin{equation}
		\left\{ h:\hat{L}\left( h,\mathbf{x}\right) \leq r\right\} =m_{r,\mathbf{x}%
		}^{-1}\left( \left( 0,m_{r,\mathbf{x}}^{\ast }\right] \right) .
		\label{Identity}
	\end{equation}
	
	(i) The function $m_{r,\mathbf{x}}$, being the result of a finite number of $%
	\max $ or $\min $ operations, is continuous on $\mathcal{H}$. It follows
	that $O=m_{r,\mathbf{x}}^{-1}\left( \left( 0,m_{r,\mathbf{x}}^{\ast }\right)
	\right) \subseteq \left\{ h:\hat{L}\left( h,\mathbf{x}\right) \leq r\right\} 
	$ is an open subset of $\mathcal{H}$. Let $\epsilon \in \left( 0,m_{r,%
		\mathbf{x}}^{\ast }\right) $, so there exists $h\in \mathcal{H}$ such that $%
	m_{r,\mathbf{x}}\left( h\right) >m_{r,\mathbf{x}}^{\ast }-\epsilon $. Choose 
	$s\in \left( -m_{r,\mathbf{x}}\left( h\right) ,0\right) \cup \left( 0,m_{r,%
		\mathbf{x}}\left( h\right) \right) $. I claim that $m_{r,\mathbf{x}}\left(
	h_{s}\right) >0$. Let $I$ be the maximizer in the definition of $m_{r,%
		\mathbf{x}}\left( h\right) $ and $i\in I$. Then $\Phi \left(
	h_{s},z_{i}\right) y_{i}=\Phi \left( h,z_{i}\right) y_{i}+sy_{i}\geq m_{r,%
		\mathbf{x}}\left( h\right) +sy_{i}>0$. By maximality $h_{s}\in m_{r,\mathbf{x%
	}}^{-1}\left( \left( 0,m_{r,\mathbf{x}}^{\ast }\right) \right) =O$, which is
	therefore also nonempty. The second assertion follows immediately from the
	first.
	
	(ii) Take $h\in \left\{ g:\hat{L}\left( g,\mathbf{x}^{\prime }\right) \leq
	r\right\} $. Then $0\leq m_{r\mathbf{x}^{\prime }}\left( h\right) \leq m_{r%
		\mathbf{x}}\left( h\right) \leq m_{r\mathbf{x}}^{\ast }$, so $h\in m_{r,%
		\mathbf{x}}^{-1}\left( \left( 0,m_{r,\mathbf{x}}^{\ast }\right] \right)
	=\left\{ g:\hat{L}\left( h,\mathbf{x}\right) \leq r\right\} $ by (\ref%
	{Identity}).
\end{proof}

\subsection{The zero-temperature limit}

The next Proposition shows, that the upper bound on $\Lambda _{\beta }\left(
h,\mathbf{x}\right) $ in Corollary \ref{Corollary general}, which does not
depend on $\beta $, is in fact the limit as $\beta \rightarrow \infty $.

\begin{proposition}
	\label{Proposition limit}Fix $\mathbf{x}\in \mathcal{X}^{n}$. If $\pi \left( 
	\widehat{\mathcal{H}}_{\min }\left( \mathbf{x}\right) \right) >0$ then $%
	\Lambda _{\beta }\left( h,\mathbf{x}\right) \rightarrow \ln \left( 1/\pi
	\left( \widehat{\mathcal{H}}_{\min }\left( \mathbf{x}\right) \right) \right) 
	$ in probability as $\beta \rightarrow \infty $.
\end{proposition}

\begin{proof}
	We already have $\Lambda _{\beta }\left( h,\mathbf{x}\right) \leq \ln \left(
	1/\pi \left( \widehat{\mathcal{H}}_{\min }\left( \mathbf{x}\right) \right)
	\right) $. For the other direction fix $\eta ,\delta >0$ and note, that by
	the right continuity of the distribution function there is $\epsilon $ such
	that for all $0<\epsilon ^{\prime }\leq \epsilon $%
	\begin{equation}
		\beta \epsilon ^{\prime }+\ln \frac{1}{\hat{\varphi}\left( \hat{L}_{\min
			}\left( \mathbf{x}\right) +2\epsilon ^{\prime },\mathbf{x}\right) }>\ln
		\left( 1/\pi \left( \widehat{\mathcal{H}}_{\min }\left( \mathbf{x}\right)
		\right) \right) -\eta .  \label{lower bound lambda}
	\end{equation}%
	By Proposition 3.1 in \cite{athreya2010gibbs} there exists $\beta _{0}$ such
	that $\beta >\beta _{0}$ implies that $\Pr_{h\sim \hat{G}_{\beta }\left( 
		\mathbf{x}\right) }\left\{ \hat{L}\left( h,\mathbf{x}\right) \leq \hat{L}%
	_{\min }\left( \mathbf{x}\right) +\epsilon \right\} \geq 1-\delta $. In this
	event 
	\begin{eqnarray*}
		\Lambda _{\beta }\left( h,\mathbf{x}\right) &=&\beta r^{\ast }-\ln \hat{%
			\varphi}\left( \hat{L}\left( h,\mathbf{x}\right) +r^{\ast },\mathbf{x}\right)
		\\
		&\geq &\beta r^{\ast }-\ln \hat{\varphi}\left( \hat{L}_{\min }\left( \mathbf{%
			x}\right) +\epsilon +r^{\ast },\mathbf{x}\right) \\
		&\geq &\beta r\left( \beta ,\epsilon \right) -\ln \hat{\varphi}\left( \hat{L}%
		_{\min }\left( \mathbf{x}\right) +\epsilon +r\left( \beta ,\epsilon \right) ,%
		\mathbf{x}\right) ,
	\end{eqnarray*}%
	where $r^{\ast }$ is the minimizer in the definition of $\Lambda $ and $%
	r\left( \beta ,\epsilon \right) $ is the minimizer of the new right hand
	side, which doesn't depend on $h$ but on $\beta $ and $\epsilon $. By
	Corollary \ref{Corollary general} 
	\begin{equation*}
		\ln \left( 1/\pi \left( \widehat{\mathcal{H}}_{\min }\left( \mathbf{x}%
		\right) \right) \right) \geq \beta r\left( \beta ,\epsilon \right) -\ln \hat{%
			\varphi}\left( \hat{L}_{\min }\left( \mathbf{x}\right) +\epsilon +r\left(
		\beta ,\epsilon \right) ,\mathbf{x}\right) .
	\end{equation*}%
	This implies $r\left( \beta ,\epsilon \right) \leq \ln \left( 1/\pi \left( 
	\widehat{\mathcal{H}}_{\min }\left( \mathbf{x}\right) \right) \right) /\beta 
	$, so by making $\beta $ large enough we can ensure that $r\left( \beta
	,\epsilon \right) <\epsilon $, so (\ref{lower bound lambda}) implies $%
	\Lambda \left( h,\mathbf{x}\right) \geq \ln \left( 1/\pi \left( \widehat{%
		\mathcal{H}}_{\min }\left( \mathbf{x}\right) \right) \right) -\eta $.
\end{proof}

\begin{figure}[h!]
	\centering
	\includegraphics[width=0.7\textwidth]{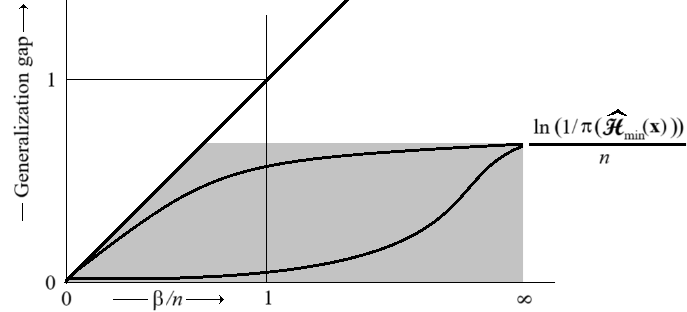}  
	\caption{Schematic and compactified phase diagram of the bounds when $\pi (\widehat{\mathcal{H}}_{\min }\left( \mathbf{x}\right))>0$ with $n$ fixed. The diagonal represents the data-independent bounds of Corollary \ref{Corollary high temperature}. The data-dependent bounds have to lie in the shaded region by Corollary \ref{Corollary general} and converge to $\ln \left( 1/\pi \left( 
		\widehat{\mathcal{H}}_{\min }\left( \mathbf{x}\right) \right) \right)/n$ by Proposition \ref{Proposition limit}, ignoring smaller logarithmic terms.}
	\label{Phase diagram}
\end{figure}

As an example let $\mathcal{H}$ be finite, with $\pi $ being the uniform
counting measure and consider the Gibbs-algorithm in the low temperature
limit $\beta \rightarrow \infty $, where the posterior becomes uniform on
the set of minimizers (see \cite{athreya2010gibbs}). If there is only a
single minimizer then $\pi \left( \widehat{\mathcal{H}}_{\min }\left( 
\mathbf{x}\right) \right) =1/\left\vert \mathcal{H}\right\vert $ and the
bound becomes one of roughly order $\ln \left( \left\vert \mathcal{H}%
\right\vert \right) /n$, which is just the usual bound, serving as a sanity
check. But if there are $K$ minimizers we get an additional term of $-\ln
K/n $, decreasing the generalization gap, another instance of cooperative
behavior. A similar phenomenon is described in \cite{langford2004computable}.

The consequences of Corollaries \ref{Corollary high temperature}, \ref%
{Corollary general} and Proposition \ref{Proposition limit} are summarized
in Figure \ref{Phase diagram}.

\subsection{Distribution-dependence and reproducibility}

With $\mathcal{H}$ and $\pi $ fixed generalization of the Gibbs algorithm
should be a property of the underlying data distribution $\mu $. We expect
that for new data drawn from $\mu ^{n}$ similar results should be obtained.
Since the bound in Theorem \ref{Theorem Main} depends essentially on the
cumulative distribution function of the empirical loss $\hat{\varphi}\left(
r,\mathbf{x}\right) =\pi \left\{ h:\hat{L}\left( h,\mathbf{x}\right)
>r\right\} $, this function should in some sense concentrate on its
distribution dependent counterpart $\varphi \left( r\right) =\pi \left\{
h:L\left( h\right) >r\right\} $. Such is the content of the following
proposition, which may be of independent interest.

\begin{proposition}
	\label{Proposition compare distribution functions}Let $\delta >0$ and $p\in 
	\mathbb{N}$. Set 
	\begin{equation*}
		s\left( n,\delta ,p\right) :=\sqrt{\frac{\ln \left( \left( 1+n^{2p+1}\right)
				/\delta \right) }{2n}}.
	\end{equation*}
	
	(i) With probability at least $1-\delta $ in $\mathbf{x}\sim \mu ^{n}$ we
	have for all $r\in \mathbb{R}$ that%
	\begin{equation*}
		\hat{\varphi}\left( r+s\left( n,\delta ,p\right) ,\mathbf{x}\right) \geq
		\varphi \left( r\right) -n^{-p}s\left( n,\delta ,p\right) .
	\end{equation*}
	
	(ii) With probability at least $1-\delta $ in $\mathbf{x}\sim \mu ^{n}$ we
	have for all $r\in \mathbb{R}$ that%
	\begin{equation*}
		\varphi \left( r+s\left( n,\delta ,p\right) \right) \geq \hat{\varphi}\left(
		r,\mathbf{x}\right) -n^{-p}s\left( n,\delta ,p\right) .
	\end{equation*}
\end{proposition}

We allow a shift within the cumulative distribution functions of $O\left(
\left( p\ln n\right) /n\right) $ but a shift of the measures smaller by a
factor of $n^{-p}$, where we can choose $p$. This is because of the
magnitudes of numbers we expect. In the numerical example in Section \ref%
{Section low temperatures} we had $r\approx 10^{-3}$, but $\hat{\varphi}%
\left( r,\mathbf{x}\right) =\pi \left\{ g:\hat{L}\left( g,\mathbf{x}\right)
\leq r\right\} \approx 10^{-44}$.

The proof (detailed in Appendix \ref{Section proof of proposition compare
	distribution functions}) first reduces the inequality in (i) to a bound on
the probability $\mu ^{n}\left\{ \mathbf{x}:\pi \left\{ h:\hat{L}\left( h,%
\mathbf{x}\right) -L\left( h\right) >s\right\} >t\right\} $. This is then
bounded by approximating $\pi $ with $\left( 1+n^{2p+1}\right) $ trials, for
each trial $h_{k}\sim \pi $ estimating $\mu ^{n}\left\{ \hat{L}\left( h_{k},%
\mathbf{x}\right) -L\left( h_{k}\right) >s\right\} $ with Hoeffding's
inequality and concluding with a union bound over the trials.

For difficult or impossible tasks, such as randomly permuted labels, $%
L_{\min }=~$ess$~\inf_{h\in \mathcal{H}}L\left( h\right) $ is large, and $%
\varphi \left( r\right) =0$ for $r<L_{\min }$. But for overparametrized $%
\mathcal{H}$ and large $\beta $ it may yet happen that $\hat{L}_{\min
}\left( \mathbf{x}\right) $ is small or even zero. Proposition \ref%
{Proposition compare distribution functions} (ii) then still guarantees with
high probability 
\begin{equation*}
	\hat{\varphi}\left( L_{\min }-t\left( n,\delta \right) ,\mathbf{x}\right)
	\leq n^{-p}s\left( n,\delta ,p\right) ,
\end{equation*}%
so for randomly permuted labels the total prior volume of hypotheses with
small empirical error is necessarily small and decreases with the sample
size, regardless of the fact that we can find small minima of the empirical
error. In this sense Theorem \textbf{\ref{Theorem Main}} and Proposition \ref%
{Proposition compare distribution functions} predict narrow minima for
random labels. Figures 7 and 8 in \cite{zhang2018theory} illustrate this
point.

With Proposition \ref{Proposition compare distribution functions} (i) at
hand a union bound gives the following corollary of Theorem \ref{Theorem
	Main}.

\begin{corollary}
	Let $F$ be a measurable function on $\mathcal{H}\times \mathcal{X}^{n}$. For 
	$\delta >0$ and $p\in \mathbb{N}$ and $\left( h,\mathbf{x}\right) \in 
	\mathcal{H}\times \mathcal{X}^{n}$ let%
	\begin{equation*}
		S\left( h,\mathbf{x}\right) =\left\{ r:\varphi \left( \hat{L}\left( h,%
		\mathbf{x}\right) +r-s\left( n,\delta ,p\right) \right) -n^{-p}s\left(
		n,\delta ,p\right) >0\right\} .
	\end{equation*}%
	Then with probability at least $1-\delta $ as $\mathbf{x}\sim \mu ^{n}$ 
	\begin{eqnarray*}
		F\left( h,\mathbf{x}\right) &\leq &\inf_{r\in S\left( h,\mathbf{x}\right)
		}\beta r+\ln \frac{1}{\varphi \left( \hat{L}\left( h,\mathbf{x}\right)
			+r-s\left( n,\delta ,p\right) \right) -n^{-p}s\left( n,\delta ,p\right) } \\
		&&+\ln E_{\mathbf{x}}E_{h\sim \pi }\left[ e^{F\left( h,\mathbf{x}\right) }%
		\right] +\ln \left( 2/\delta \right) .
	\end{eqnarray*}
\end{corollary}

The fact that Proposition \ref{Proposition compare distribution functions}
allows a shift of $O\left( \sqrt{\left( p\ln n\right) /n}\right) $ makes the
distribution-dependent bound above somewhat loose for the important small
values of $r$.

\section{Other corollaries of Theorem \textbf{\protect\ref{Theorem Main}} 
	\label{Section Other bounds}}

The freedom in the choice of $F$ allows a number of bounds to be derived
from Theorem \textbf{\ref{Theorem Main}}. A real random variable $Y$ is $%
\sigma $-sub-Gaussian if $\ln \mathbb{E}e^{\lambda Y-\mathbb{E}Y}\leq
\lambda ^{2}\sigma ^{2}/2$ for all $\lambda \in \mathbb{R}$. Now suppose all
the $x\in \mathcal{X\mapsto \ell }\left( h,x\right) $ are $\sigma $%
-sub-Gaussian as $x\sim \mu $. Then $\mathbf{x}\in \mathcal{X}^{n}\mapsto 
\hat{L}\left( h,\mathbf{x}\right) $ as $\mathbf{x}\sim \mu ^{n}$ is $\sigma /%
\sqrt{n}$-sub-Gaussian. It is tempting to set $F=\lambda \left( L\left(
h\right) -\hat{L}\left( h,\mathbf{x}\right) \right) $ in Theorem \ref%
{Theorem Main}, divide by $\lambda $ and then optimize over $\lambda $.
Unfortunately the last step is impossible, since the optimal $\lambda $ is
data-dependent in its dependence on $\Lambda $ and ruins the exponential
moment bound on $F$. A more careful argument given in Appendix \ref{Appendix
	stratify} stratifies the values of $\Lambda $ and establishes the following.

\begin{corollary}
	\label{Corollary Stratify}Suppose that for all $h\in \mathcal{H}$ the random
	variables $x\in \mathcal{X\mapsto \ell }\left( h,x\right) $ as $x\sim \mu $
	are $\sigma $-sub-Gaussian. Then for $\delta >0$ with probability at least $%
	1-\delta $ as $\left( h,\mathbf{x}\right) \sim \rho $%
	\begin{equation*}
		\left\vert L\left( h\right) -\hat{L}\left( h,\mathbf{x}\right) \right\vert
		\leq 2\sigma \sqrt{\frac{\max \left\{ \Lambda _{\beta }\left( h,\mathbf{x}%
				\right) ,1\right\} +\ln \left( 2\Lambda _{\beta }\left( h,\mathbf{x}\right)
				/\delta \right) /2}{n}}.
	\end{equation*}
\end{corollary}

Similar techniques lead to bounds for sub-exponential losses. Here we only
give a weak bound with the following direct and crude argument. Set $F\left(
h,\mathbf{x}\right) =\sqrt{n}\left( L\left( h\right) -\hat{L}\left( h,%
\mathbf{x}\right) \right) $ and use Proposition 2.7.1 and (2.24) in \cite%
{vershynin2018high} with $\lambda =\sqrt{n}$ to obtain the following.

\begin{corollary}
	If $\delta >0$ then there exist absolute constants $c_{1}$ and $c_{2}$ such
	that for $\sqrt{n}\geq c_{2}\sup_{g\in \mathcal{H}}\left\Vert
	Y_{g}\right\Vert _{\psi _{1}}$ and $\delta >0$ with probability at least $%
	1-\delta $ as $X\sim \mu ^{n}$ and $h\sim \hat{G}_{\beta }\left( \mathbf{x}%
	\right) $ 
	\begin{equation*}
		L\left( h\right) -\hat{L}\left( h,\mathbf{x}\right) \leq \frac{\Lambda
			_{\beta }\left( h,\mathbf{x}\right) +c_{1}\sup_{g\in \mathcal{H}}\left\Vert
			Y_{g}\right\Vert _{\psi _{1}}+\ln \left( 1/\delta \right) }{\sqrt{n}},
	\end{equation*}%
	where $Y_{g}$ is the random variable $x\in \mathcal{X\mapsto \ell }\left(
	g,x\right) $ as $x\sim \mu $.
\end{corollary}

The quantity $\Lambda _{\beta }\left( h,\mathbf{x}\right) $ is oblivious to
the nature of the random variable $\mathbf{x}$, so any method to bound $\ln 
\mathbb{E}_{\mathbf{x}}e^{\lambda \left( L\left( h\right) -\hat{L}\left( h,%
	\mathbf{x}\right) \right) }$ can be used to derive bounds from Theorem 
\textbf{\ref{Theorem Main}}, as long as either $\lambda $ is fixed
beforehand ($\sqrt{n}$ seems always a good choice), or special care is taken
as in the proof of Corollary \ref{Corollary Stratify}, which leads to an
additional logarithmic term in $\Lambda _{\beta }\left( h,\mathbf{x}\right) $%
. In this way we obtain bounds also for martingales or complicated nonlinear
functions of the data, whose exponential moments can be controlled.

Theorem \textbf{\ref{Theorem Main}} highlights the benefit of a well aligned
prior reference distribution $\pi $. Since the bound is data-dependent this
suggests the use of a data-dependent prior $\pi \left( \mathbf{x}\right) $.
Then the expectations in $\ln \mathbb{E}_{\mathbf{x}}\mathbb{E}_{h\sim \pi
	\left( \mathbf{x}\right) }\left[ e^{F\left( h,\mathbf{X}\right) }\right] $
cannot be exchanged, and the situation becomes more complicated. But several
solutions are given in \cite{dziugaite2018data}, \cite{rivasplata2020pac}
and \cite{maurer2024generalization}. Bounds on $\ln \mathbb{E}_{\mathbf{X}}%
\mathbb{E}_{h\sim \pi \left( \mathbf{x}\right) }\left[ e^{F\left( h,\mathbf{X%
	}\right) }\right] $ exist and can be substituted in Theorem \textbf{\ref%
	{Theorem Main}}, for example when the prior is given by Gaussian
randomization of a stable algorithm as described in Section 4 of the last
reference above.

\section{Beyond the Gibbs algorithm\label{Section beyond Gibbs}}

Theorem \textbf{\ref{Theorem Main}} can be extended to other stochastic
algorithms, if they produce densities, which are non-increasing functions of
the empirical loss, satisfying a logarithmic Lipschitz condition.

\begin{theorem}
	\label{Theorem Beyond Gibbs}Suppose that there is a measurable function $q:%
	\left[ 0,\infty \right) \times \mathcal{X}^{n}\rightarrow \left[ 0,\infty
	\right) $ such that for every $\mathbf{x}\in \mathcal{X}^{n}$
	
	(i) $\int_{\mathcal{H}}q\left( \hat{L}\left( h,\mathbf{x}\right) ,\mathbf{x}%
	\right) d\pi \left( h\right) =1$.
	
	(ii) $q\left( t,\mathbf{x}\right) $ is non-increasing in $t\in \left[
	0,\infty \right) $.
	
	(iii) $\ln q\left( t,\mathbf{x}\right) -\ln q\left( s,\mathbf{x}\right) \leq
	\gamma \left( \mathbf{x}\right) \left\vert t-s\right\vert $ for $s,t\in %
	\left[ 0,\infty \right) $.
	
	Let $Q\left( \mathbf{x}\right) $ be the measure on $\Omega $ defined by $%
	Q\left( \mathbf{x}\right) \left( A\right) =\int_{A}q\left( \hat{L}\left( h,%
	\mathbf{x}\right) ,\mathbf{x}\right) d\pi \left( h\right) $. Then $Q\left( 
	\mathbf{x}\right) \in \mathcal{P}\left( \mathcal{H}\right) $ and for $F$ as
	in Theorem \textbf{\ref{Theorem Main}} and $\delta >0$ with probability at
	least $1-\delta $ as $x\sim \mu ^{n}$ and $h\sim Q\left( \mathbf{x}\right) $%
	\begin{equation*}
		F\left( h,\mathbf{x}\right) \leq \Lambda _{\gamma \left( \mathbf{x}\right)
		}\left( h,\mathbf{x}\right) +\ln \mathbb{E}_{\mathbf{x}}\mathbb{E}_{g\sim
			\pi }\left[ e^{F\left( g,\mathbf{x}\right) }\right] +\ln \left( 1/\delta
		\right) .
	\end{equation*}
\end{theorem}

The Gibbs algorithm satisfies the conditions of the theorem with $\gamma
\left( \mathbf{x}\right) =\beta $. One conclusion is that Theorem \textbf{%
	\ref{Theorem Main}} also holds with data-dependent temperature, but Theorem %
\ref{Theorem Beyond Gibbs} gives much greater freedom in the choice of
posteriors. The proof, detailed in Appendix \ref{Proof of Theorem beyond
	Gibbs}, is similar to the proof of Theorem \textbf{\ref{Theorem Main}}.

\section{Related work\label{Section related work}}

The Gibbs algorithm traces its origin to the work of \cite%
{boltzmann1877beziehung} and \cite{gibbs1902elementary} on statistical
mechanics, and its relevance to machine learning was recognized by \cite%
{levin1990statistical} and \cite{opper1991calculation}. \cite%
{mcallester1999pac} realized that the minimizers of the PAC-Bayesian bound
are Gibbs distributions. The fact that they are limiting distributions of
stochastic gradient Langevin dynamics (\cite{raginsky2017non}), raises the
question about the generalization properties of individual hypotheses as
addressed in this paper. Average generalization of the Gibbs posterior was
further studied notably by \cite{aminian2021exact} and \cite%
{aminian2023information}, where there are also investigations into the
limiting behavior as $\beta \rightarrow \infty $.

Theorem \textbf{\ref{Theorem Main}} is part of the circle of information
theoretic ideas in machine learning, ranging from the PAC-Bayesian theorem (%
\cite{shawe1997pac}, \cite{mcallester1999pac}, \cite{mcallester2003pac}, 
\cite{catoni2003pac}) to generalization bounds in terms of mutual
information (\cite{russo2016controlling} and \cite{xu2017information}). It
is inspired by and indebted to the disintegrated PAC-Bayesian bounds as in 
\cite{blanchard2007occam}, \cite{rivasplata2020pac} and \cite%
{viallard2024general}.

The benefit of wide minima was noted by \cite{hochreiter1997flat}, where
also a variant of the Gibbs algorithm was discussed. The idea was promoted
by \cite{keskar2016large} and others \cite{zhang2018theory}, \cite%
{iyer2023wide}. It was soon objected by \cite{dinh2017sharp} that there are
narrow reparametrizations of wide minima which compute the same function.
Several authors then searched for reparametrization-invariant measures of
"width" (\cite{andriushchenko2023modern}, \cite{kristiadi2024geometry}).
Nevertheless it was early conjectured (\cite{neyshabur2017exploring}), that
the relevant property is average width, which is also the position of the
paper at hand.

\section{Conclusion and future directions}

The principal contributions of this paper are the application of
disintegrated PAC-Bayesian bounds to the Gibbs algorithm and the lower bound
on the partition function, which exposes the connection between
generalization and the cumulative distribution function of the empirical
loss.

Recent studies of the loss landscapes of overparametrized non-convex systems
suggest, that global minimizers are generically high-dimensional manifolds (%
\cite{cooper2018loss}, \cite{cooper2021global}, \cite{liu2022loss}). An
interesting future research direction is to investigate the behaviour of the
empirical loss in the neighborhood of these manifolds. 

Another project is the search for lower bounds. Clearly the Gibbs posterior
is unlikely to sample good hypotheses if they are very scarce, and for
bounded losses there is an elementary lower bound on the probability of
sampling a hypothesis with a given upper bound $r$ on the true error. But
this lower bound requires $\varphi \left( r\right) $ to be exponentially
small in $\beta $, which is very far from the bound in Theorem \ref{Theorem
	Main}.

The most important future work is to find a way to obtain quantitative
confirmation of the qualitative predictions made by the paper.

\bibliographystyle{abbrvnat}

\newpage
\appendix

\section{Table of notation\protect\bigskip}

,%
\begin{tabular}{l|l}
	\hline
	$\mathcal{X}$ & space of data \\ \hline
	$\mu $ & probability of data \\ \hline
	$n$ & sample size \\ \hline
	$\mathbf{x}$ & generic member $\left( x_{1},...,x_{n}\right) \in \mathcal{X}%
	^{n}$ \\ \hline
	$\mathbf{x}$ & training set $\mathbf{x}=\left( X_{1},...,X_{n}\right) \sim
	\mu ^{n}$ \\ \hline
	$\mathcal{H}$ & hypothesis space$:\mathcal{X\rightarrow }\left[ 0,\infty
	\right) $) \\ \hline
	$h,g$ & members of $\mathcal{H}$ \\ \hline
	$\ell $ & loss function $\ell :\mathcal{H\times X}\rightarrow \left[
	0,\infty \right) $ \\ \hline
	$h\left( x\right) ,g\left( x\right) $ & shorthand for $\ell \left(
	h,x\right) ,\ell \left( g,x\right) $ \\ \hline
	$\mathcal{P}\left( \mathcal{H}\right) $ & probability measures on $\mathcal{H%
	}$ \\ \hline
	$\pi $ & prior reference measure on $\mathcal{H}$ \\ \hline
	$L\left( h\right) $ & $L\left( h\right) =\mathbb{E}_{x\sim \mu }\left[
	h\left( x\right) \right] =\mathbb{E}_{x\sim \mu }\left[ \ell \left(
	h,x\right) \right] $, expected (true) risk of $h\in \mathcal{H}$ \\ \hline
	$\hat{L}\left( h,\mathbf{x}\right) $ & $\hat{L}\left( h,\mathbf{x}\right)
	=\left( 1/n\right) \sum_{i=1}^{n}h\left( X_{i}\right) =\left( 1/n\right)
	\sum_{i=1}^{n}\ell \left( h,X_{i}\right) $, empirical risk of $h\in \mathcal{%
		H}$ \\ \hline
	$L_{\min }$ & $L_{\min }=$ ess$~\inf_{h\in \mathcal{H}}L\left( h\right) $,
	global risk minimum \\ \hline
	$\hat{L}_{\min }\left( \mathbf{x}\right) $ & $\hat{L}_{\min }\left( \mathbf{x%
	}\right) =$ ess$~\inf_{h\in \mathcal{H}}L\left( h,\mathbf{x}\right) $,
	global empirical risk minimum \\ \hline
	$\mathcal{H}_{\min }$ & $\mathcal{H}_{\min }=\left\{ h:L\left( h\right)
	=L_{\min }\right\} $, set of risk minimizers \\ \hline
	$\widehat{\mathcal{H}}_{\min }\left( \mathbf{x}\right) $ & $\widehat{%
		\mathcal{H}}_{\min }\left( \mathbf{x}\right) =\left\{ h:L\left( h,\mathbf{x}%
	\right) =\hat{L}_{\min }\left( \mathbf{x}\right) \right\} $, set of
	empirical risk minimizers \\ \hline
	$\varphi \left( r\right) $ & $\varphi \left( r\right) =\pi \left\{ g:L\left(
	g\right) \leq r\right\} $, cumulative distribution function of true loss \\ 
	\hline
	$\hat{\varphi}\left( r,\mathbf{x}\right) $ & $\hat{\varphi}\left( r,\mathbf{x%
	}\right) =\pi \left\{ g:\hat{L}\left( g,\mathbf{x}\right) \leq r\right\} $,
	cumulative distribution function of empirical loss \\ \hline
	$\beta $ & inverse temperature \\ \hline
	$Z_{\beta }\left( \mathbf{x}\right) $ & $Z_{\beta }\left( \mathbf{x}\right)
	=\int_{\mathcal{H}}e^{-\beta \hat{L}\left( h,\mathbf{x}\right) }d\pi \left(
	h\right) $, partition function \\ \hline
	$\hat{G}_{\beta }\left( \mathbf{x}\right) $ & $\hat{G}_{\beta }\left( 
	\mathbf{x}\right) =Z_{\beta }\left( \mathbf{x}\right) ^{-1}e^{-\beta \hat{L}%
		\left( h,\mathbf{x}\right) }d\pi \left( h\right) $, Gibbs posterior \\ \hline
	$\rho $ & $\rho \left( A\right) =\mathbb{E}_{\mathbf{x}}\mathbb{E}_{h\sim 
		\hat{G}_{\beta }\left( \mathbf{x}\right) }\left[ 1_{A}\left( h,\mathbf{x}%
	\right) \right] $, joint distribution of $\mathbf{x}$ and $\hat{G}_{\beta
	}\left( \mathbf{x}\right) $ on $\mathcal{H\times X}^{n}$ \\ \hline
	$\Lambda _{\beta }\left( h,\mathbf{x}\right) $ & $\Lambda _{\beta }\left( h,%
	\mathbf{x}\right) =\inf_{r\in \mathbb{R}}\beta r+\ln \frac{1}{\pi \left\{ g:%
		\hat{L}\left( g,\mathbf{x}\right) \leq \hat{L}\left( h,\mathbf{x}\right)
		+r\right\} }$, complexity measure \\ \hline
	$\kappa $ & $\kappa \left( p,q\right) =p\ln \frac{p}{q}+\left( 1-p\right)
	\ln \frac{1-p}{1-q}$, relative entropy of $p$- and $q$-Bernoulli variables
	\\ \hline
	$\left\vert A\right\vert $ & cardinality of set \\ \hline
	$1_{A}$ & indicator function of set \\ \hline
	$\left\Vert .\right\Vert _{\psi _{2}}$ & sub-Gaussian norm (see Sec. 2.5.2
	in \cite{vershynin2018high}) \\ \hline
	$\left\Vert .\right\Vert _{\psi _{1}}$ & sub-exponential norm (see Sec. 2.7
	in \cite{vershynin2018high}) \\ \hline
\end{tabular}
\bigskip 

\newpage

\section{Markov's inequality and integral probability metrics\label{Section
		Markov inequality}}

\begin{definition}
	\label{Definition Integral probability metric}Let $\left( \mathcal{Y},\Xi
	\right) $ be a measurable space. If $\mathcal{F}$ is a set of measurable
	real valued functions on $\left( \mathcal{Y},\Xi \right) $ the integral
	probability metric is the metric $\Delta _{\mathcal{F}}$\ on $\mathcal{P}%
	\left( \mathcal{Y}\right) $ defined by 
	\begin{equation*}
		\Delta _{\mathcal{F}}\left( \zeta ,\xi \right) =\sup_{f\in \mathcal{F}%
		}\left\vert \mathbb{E}_{y\sim \zeta }f\left( y\right) -\mathbb{E}_{y\sim \xi
		}f\left( y\right) \right\vert .
	\end{equation*}
\end{definition}

\begin{lemma}
	\label{Markov inequality}(i) Let $Y$ be real random variable and $\delta >0$%
	. Then%
	\begin{equation*}
		\Pr \left\{ Y>\ln \mathbb{E}\left[ e^{Y}\right] +\ln \left( 1/\delta \right)
		\right\} <\delta .
	\end{equation*}
	
	(ii) Let $f$ be a measurable real function on $\left( \mathcal{Y},\Xi
	\right) $, $\mathcal{F}$ ia set of measurable real valued functions on $%
	\left( \mathcal{Y},\Xi \right) $ and $\zeta $,$\xi \in \mathcal{P}\left( 
	\mathcal{Y}\right) $. Then%
	\begin{equation*}
		\Pr_{y\sim \zeta }\left\{ f\left( y\right) >\ln \mathbb{E}_{y\sim \xi }\left[
		e^{f\left( y\right) }\right] +\sup \left\{ \gamma :\frac{e^{f\left( y\right)
		}}{\gamma }\in \mathcal{F}\right\} \Delta _{\mathcal{F}}\left( \zeta ,\xi
		\right) +\ln \left( 1/\delta \right) \right\} .
	\end{equation*}
\end{lemma}

\begin{proof}
	(i) $\Pr \left\{ Y>\ln \mathbb{E}\left[ e^{Y}\right] +\ln \left( 1/\delta
	\right) \right\} =\Pr \left\{ e^{Y}>\frac{\mathbb{E}\left[ e^{Y}\right] }{%
		\delta }\right\} <\frac{\mathbb{E}\left[ e^{Y}\right] }{\mathbb{E}\left[
		e^{Y}\right] /\delta }=\delta $, where the inequality is just Markov's
	inequality in its usual form.
\end{proof}

(ii) We have 
\begin{equation*}
	\mathbb{E}_{y\sim \zeta }\left[ e^{f\left( y\right) }\right] \leq \mathbb{E}%
	_{y\sim \xi }\left[ e^{f\left( y\right) }\right] +\sup \left\{ \gamma :\frac{%
		e^{f\left( y\right) }}{\gamma }\in \mathcal{F}\right\} \Delta _{\mathcal{F}%
	}\left( \zeta ,\xi \right) .
\end{equation*}%
Using $\ln \left( a+b\right) \leq \ln a+\frac{b}{a}$ for $a,b>0$ we get 
\begin{equation*}
	\ln \mathbb{E}_{y\sim \zeta }\left[ e^{f\left( y\right) }\right] \leq \ln 
	\mathbb{E}_{y\sim \xi }\left[ e^{f\left( y\right) }\right] +\frac{\sup
		\left\{ \gamma :\frac{e^{f\left( y\right) }}{\gamma }\in \mathcal{F}\right\} 
	}{\mathbb{E}_{y\sim \xi }\left[ e^{f\left( y\right) }\right] }\Delta _{%
		\mathcal{F}}\left( \zeta ,\xi \right) .
\end{equation*}%
Then use (i).

\section{Proof of Proposition \protect\ref{Proposition compare distribution
		functions}\label{Section proof of proposition compare distribution functions}%
}

\begin{proposition}[Restatement of Proposition \protect\ref{Proposition
		compare distribution functions}]
	Let $\delta >0$ and $p\in \mathbb{N}$. Set%
	\begin{equation*}
		s\left( n,\delta ,p\right) :=\sqrt{\frac{\ln \left( \left( 1+n^{2p+1}\right)
				/\delta \right) }{2n}}.
	\end{equation*}
	
	(i) With probability at least $1-\delta $ in $\mathbf{x}\sim \mu ^{n}$ we
	have for all $r\in \mathbb{R}$ that%
	\begin{equation*}
		\hat{\varphi}\left( r+s\left( n,\delta ,p\right) ,\mathbf{x}\right) \geq
		\varphi \left( r\right) -n^{-p}s\left( n,\delta ,p\right) .
	\end{equation*}
	
	(ii) With probability at least $1-\delta $ in $\mathbf{x}\sim \mu ^{n}$ we
	have for all $r\in \mathbb{R}$ that%
	\begin{equation*}
		\varphi \left( r+s\left( n,\delta ,p\right) \right) \geq \hat{\varphi}\left(
		r,\mathbf{x}\right) -n^{-p}s\left( n,\delta ,p\right) .
	\end{equation*}
\end{proposition}

\begin{proof}
	Let $s,t>0$. Then, writing out the probability measure on $\mathcal{X}^{n}$
	as $\mu ^{n}$, we have for all $r\in \mathbb{R}$%
	\begin{align*}
		& \Pr_{\mathbf{x}\sim \mu ^{n}}\left\{ \hat{\varphi}\left( r+s,\mathbf{x}%
		\right) \geq \varphi \left( r\right) -t\right\}  \\
		& =\mu ^{n}\left\{ \mathbf{x}:\pi \left\{ h:\hat{L}\left( h,\mathbf{x}%
		\right) \leq r+s\right\} <\pi \left\{ L\left( h\right) \leq r\right\}
		-t\right\}  \\
		& \leq \mu ^{n}\left\{ \mathbf{x}:\pi \left\{ h:\hat{L}\left( h,\mathbf{x}%
		\right) \leq r+s\wedge L\left( h\right) \leq r\right\} <\pi \left\{ L\left(
		h\right) <r\right\} -t\right\}  \\
		& =\mu ^{n}\left\{ \mathbf{x}:\pi \left\{ L\left( h\right) \leq r\right\}
		-\pi \left\{ h:\hat{L}\left( h,\mathbf{x}\right) >r+s\wedge L\left( h\right)
		\leq r\right\} <\pi \left\{ L\left( h\right) \leq r\right\} -t\right\}  \\
		& =\mu ^{n}\left\{ \mathbf{x}:\pi \left\{ h:\hat{L}\left( h,\mathbf{x}%
		\right) >r+s\wedge L\left( h\right) \leq r\right\} >t\right\}  \\
		& \leq \mu ^{n}\left\{ \mathbf{x}:\pi \left\{ h:\hat{L}\left( h,\mathbf{x}%
		\right) -L\left( h\right) >s\right\} >t\right\} .
	\end{align*}%
	Note that $r$ has disappeared from the last expression. Now let $K\in 
	\mathbb{N}$ and use the fact that $\pi $ is a probability measure and
	introduce an iid $K$-sample of hypotheses $\mathbf{h}\sim \pi ^{K}$ to
	approximate $\pi $.%
	\begin{align*}
		& \mu ^{n}\left\{ \mathbf{x}:\pi \left\{ h:\hat{L}\left( h,\mathbf{x}\right)
		-L\left( h\right) >s\right\} >t\right\}  \\
		& =\pi ^{K}\times \mu ^{n}\left\{ \left( \mathbf{h},\mathbf{x}\right) :\pi
		\left\{ h:\hat{L}\left( h,\mathbf{x}\right) -L\left( h\right) >s\right\}
		>t\right\}  \\
		& \leq \pi ^{K}\times \mu ^{n}\left\{ \left( \mathbf{h},\mathbf{x}\right)
		:\pi \left\{ h:\hat{L}\left( h,\mathbf{x}\right) -L\left( h\right)
		>s\right\} -\frac{1}{K}\left\vert \left\{ k:\hat{L}\left( h_{k},\mathbf{x}%
		\right) -L\left( h_{k}\right) >s\right\} \right\vert >t\right\}  \\
		& +\pi ^{K}\times \mu ^{n}\left\{ \left( \mathbf{h},\mathbf{x}\right) :\frac{%
			1}{K}\left\vert \left\{ k:\hat{L}\left( h_{k},\mathbf{x}\right) -L\left(
		h_{k}\right) >s\right\} \right\vert >0\right\}  \\
		& \leq e^{-2Kt^{2}}+Ke^{-2ns^{2}}.
	\end{align*}%
	The first inequality is a union bound. Then the first probability is bounded
	with Hoeffding's inequality applied to $\mathbf{h}$, which gives $e^{-Kt^{2}}
	$. The event in the second probability is contained in the union of $K$
	events $\left\{ \hat{L}\left( h_{k},\mathbf{x}\right) -L\left( h_{k}\right)
	>s\right\} $ and is bounded by $Ke^{-2ns^{2}}$ using Hoeffding's inequality
	applied to $\mathbf{x}$ in combination with a union bound. The argument
	depends crucially on the independence of $\mu $ and $\pi $. Combining the
	previous two displays gives 
	\begin{equation*}
		\Pr_{\mathbf{x}}\left\{ \pi \left\{ h:\hat{L}\left( h,\mathbf{x}\right) \leq
		r+s\right\} <\pi \left\{ L\left( h\right) \leq r\right\} -t\right\} \leq
		e^{-2Kt^{2}}+Ke^{-2ns^{2}}.
	\end{equation*}%
	Now we set $K=n^{2p+1}$, and set $t=n^{-p}s$. Then the probability above
	becomes $\left( 1+n^{2p+1}\right) e^{-2ns^{2}}$ and equating it to $\delta $
	and solving for $s$ gives $s=s\left( n,\delta ,p\right) $ and $%
	t=n^{-p}s\left( n,\delta ,p\right) $. This completes the proof of (i).
	Exchanging the roles of $L\left( h\right) $ and $\hat{L}\left( h,\mathbf{x}%
	\right) $ in this argument gives (ii).
\end{proof}

\section{Proof of Corollary \protect\ref{Corollary Stratify}\label{Appendix
		stratify}}

We reproduce Lemma 15.6 in \cite{Anthony99}.

\begin{lemma}
	\label{Lemma model selection}(Lemma 15.6 in \cite{Anthony99}) Suppose $\Pr $
	is a probability distribution and%
	\begin{equation*}
		\left\{ E\left( \alpha _{1},\alpha _{2},\delta \right) :0<\alpha _{1},\alpha
		_{2},\delta \leq 1\right\}
	\end{equation*}%
	is a set of events, such that
	
	(i) For all $0<\alpha \leq 1$ and $0<\delta \leq 1$,%
	\begin{equation*}
		\Pr \left\{ E\left( \alpha ,\alpha ,\delta \right) \right\} \leq \delta .
	\end{equation*}
	
	(ii) For all $0<\alpha _{1}\leq \alpha \leq \alpha _{2}\leq 1$ and $0<\delta
	_{1}\leq \delta \leq 1$%
	\begin{equation*}
		E\left( \alpha _{1},\alpha _{2},\delta _{1}\right) \subseteq E\left( \alpha
		,\alpha ,\delta \right) .
	\end{equation*}%
	Then for $0<a,\delta <1$,%
	\begin{equation*}
		\Pr \bigcup_{\alpha \in \left( 0,1\right] }E\left( \alpha a,\alpha ,\delta
		\alpha \left( 1-a\right) \right) \leq \delta .
	\end{equation*}
\end{lemma}

\begin{proof}[Proof of Corollary \protect\ref{Corollary Stratify}]
	By a standard subgaussian bound for iid random variables we have 
	\begin{equation*}
		\ln \mathbb{E}_{\mathbf{x}}\mathbb{E}_{h\sim \pi }\left[ e^{\lambda \left(
			L\left( h\right) -\hat{L}\left( h,\mathbf{x}\right) \right) }\right] =\ln 
		\mathbb{E}_{h\sim \pi }\mathbb{E}_{\mathbf{x}}\left[ e^{\lambda \left(
			L\left( h\right) -\hat{L}\left( h,\mathbf{x}\right) \right) }\right] \leq 
		\frac{\lambda ^{2}\sigma ^{2}}{2n}.
	\end{equation*}%
	For any $\alpha \in \left( 0,1\right] $ set $\lambda \left( \alpha \right) =%
	\sqrt{2n\left( \alpha ^{-1}+\ln \left( 1/\delta \right) \right) }/\sigma $
	and define the event 
	\begin{eqnarray*}
		E\left( \alpha _{1},\alpha _{2},\delta \right) &=&\left\{ \Lambda \left( h,%
		\mathbf{x}\right) \leq \alpha _{2}^{-1}\wedge \lambda \left( \alpha
		_{1}\right) \left( L\left( h\right) -\hat{L}\left( h,\mathbf{x}\right)
		\right) >\alpha _{1}^{-1}+\frac{\lambda \left( \alpha _{1}\right) ^{2}\sigma
			^{2}}{2n}+\ln \left( 1/\delta \right) \right\} \\
		&=&\left\{ \Lambda \left( h,\mathbf{x}\right) \leq \alpha _{2}^{-1}\wedge
		L\left( h\right) -\hat{L}\left( h,\mathbf{x}\right) >\sigma \sqrt{2\left( 
			\frac{\alpha _{1}^{-1}+\ln \left( 1/\delta \right) }{n}\right) }\right\} ,
	\end{eqnarray*}%
	where the second identity is obtained by division by $\lambda \left( \alpha
	_{1}\right) $ and substitution of its value. By the first line and Theorem 
	\textbf{\ref{Theorem Main}} this set of events satisfies (i) of Lemma \ref%
	{Lemma model selection}, and it is easy to verify that it also satisfies
	(ii). Then we use $a=1/2$ and the conclusion of Lemma \ref{Lemma model
		selection} gives after some simplifications 
	\begin{equation*}
		\Pr_{\left( h,\mathbf{x}\right) \sim \rho }\left\{ L\left( h\right) -\hat{L}%
		\left( h,\mathbf{x}\right) >2\sigma \sqrt{\frac{\max \left\{ \Lambda \left(
				h,\mathbf{x}\right) ,1\right\} +\ln \left( 2\Lambda \left( h,\mathbf{x}%
				\right) /\delta \right) /2}{n}}\right\} \leq \delta .
	\end{equation*}%
	A union bound with the same inequality for $\hat{L}\left( h,\mathbf{x}%
	\right) -L\left( h\right) $ concludes the proof.
\end{proof}

\section{Proof of Theorem \protect\ref{Theorem Beyond Gibbs}\label{Proof of
		Theorem beyond Gibbs}}

\begin{theorem}[Restatement of Theorem \protect\ref{Theorem Beyond Gibbs}]
	Suppose that there is a measurable function $q:\left[ 0,\infty \right)
	\times \mathcal{X}^{n}\rightarrow \left[ 0,\infty \right) $ such that for
	every $\mathbf{x}\in \mathcal{X}^{n}$
	
	(i) $\int_{\mathcal{H}}q\left( \hat{L}\left( h,\mathbf{x}\right) ,\mathbf{x}%
	\right) d\pi \left( h\right) =1$.
	
	(ii) $q\left( t,\mathbf{x}\right) $ is nonincreasing in $t\in \left[
	0,\infty \right) $.
	
	(iii) $\ln q\left( t,\mathbf{x}\right) -\ln q\left( s,\mathbf{x}\right) \leq
	\gamma \left( \mathbf{x}\right) \left\vert t-s\right\vert $ for $s,t\in %
	\left[ 0,\infty \right) $.
	
	Let $Q\left( \mathbf{x}\right) $ be the measure on $\Omega $ defined by $%
	Q\left( \mathbf{x}\right) \left( A\right) =\int_{A}q\left( \hat{L}\left( h,%
	\mathbf{x}\right) ,\mathbf{x}\right) d\pi \left( h\right) $. Then $Q\left( 
	\mathbf{x}\right) \in \mathcal{P}\left( \mathcal{H}\right) $ and for $F$ as
	in Theorem \textbf{\ref{Theorem Main}} and $\delta >0$ with probability at
	least $1-\delta $ as $x\sim \mu ^{n}$ and $h\sim Q\left( \mathbf{x}\right) $%
	\begin{equation*}
		F\left( h,\mathbf{x}\right) \leq \Lambda _{\gamma \left( \mathbf{x}\right)
		}\left( h,\mathbf{x}\right) +\ln \mathbb{E}_{\mathbf{x}}\mathbb{E}_{g\sim
			\pi }\left[ e^{F\left( g,\mathbf{x}\right) }\right] +\ln \left( 1/\delta
		\right) .
	\end{equation*}
\end{theorem}

\begin{proof}
	By (i) $Q\left( \mathbf{x}\right) $ is a probability measure. Markov's
	inequality applied to $F\left( h,\mathbf{x}\right) -\ln q\left( \hat{L}%
	\left( h,\mathbf{x}\right) ,\mathbf{x}\right) $ gives with probability at
	least $1-\delta $ in $\mathbf{x}\sim \mu ^{n}$ and $h\sim Q\left( \mathbf{x}%
	\right) $%
	\begin{equation}
		F\left( h,\mathbf{x}\right) \leq \ln q\left( \hat{L}\left( h,\mathbf{x}%
		\right) ,\mathbf{x}\right) +\ln \mathbb{E}_{\mathbf{x}}\mathbb{E}_{g\sim \pi
		}\left[ e^{F\left( g,\mathbf{x}\right) }\right] +\ln \left( 1/\delta \right)
		.  \label{Ungl}
	\end{equation}%
	By (i) and (ii) we have for any $r$%
	\begin{eqnarray*}
		1 &=&\int_{\mathcal{H}}q\left( \hat{L}\left( g,\mathbf{x}\right) ,\mathbf{x}%
		\right) d\pi \left( g\right) \geq \int_{\left\{ g:\hat{L}\left( g,\mathbf{x}%
			\right) \leq \hat{L}\left( h,\mathbf{x}\right) +r\right\} }q\left( \hat{L}%
		\left( g,\mathbf{x}\right) ,\mathbf{x}\right) d\pi \left( g\right)  \\
		&\geq &q\left( \hat{L}\left( h,\mathbf{x}\right) +r,\mathbf{x}\right) \hat{%
			\varphi}\left( \hat{L}\left( h,\mathbf{x}\right) +r,\mathbf{x}\right) ,
	\end{eqnarray*}%
	so $\ln q\left( \hat{L}\left( h,\mathbf{x}\right) +r,\mathbf{x}\right) \leq
	-\ln \hat{\varphi}\left( \hat{L}\left( h,\mathbf{x}\right) +r,\mathbf{x}%
	\right) $. Thus%
	\begin{eqnarray*}
		\ln q\left( \hat{L}\left( h,\mathbf{x}\right) ,\mathbf{x}\right)  &=&\ln
		q\left( \hat{L}\left( h,\mathbf{x}\right) ,\mathbf{x}\right) -\ln q\left( 
		\hat{L}\left( h,\mathbf{x}\right) +r,\mathbf{x}\right) +\ln q\left( \hat{L}%
		\left( h,\mathbf{x}\right) +r,\mathbf{x}\right)  \\
		&\leq &\gamma \left( \mathbf{x}\right) r+\ln q\left( \hat{L}\left( h,\mathbf{%
			x}\right) +r,\mathbf{x}\right)  \\
		&\leq &\gamma \left( \mathbf{x}\right) r+\ln \frac{1}{\hat{\varphi}\left( 
			\hat{L}\left( h,\mathbf{x}\right) +r,\mathbf{x}\right) }.
	\end{eqnarray*}%
	Taking the infimum in $r$ and substitution in (\ref{Ungl}) complete the
	proof.
\end{proof}

\end{document}